\documentclass[a4paper,conference]{IEEEtran}

\usepackage{graphicx}

\usepackage{algorithm}
\usepackage{algorithmic}

\usepackage{xcolor}
\definecolor{ForestGreen}{rgb}{0.13, 0.55, 0.13}
\usepackage{pifont}
\newcommand{\cmark}{{\color{ForestGreen} \ding{51}}}%
\newcommand{\xmark}{{\color{red} \ding{55}}}%

\newcommand{\citet}[1]{\cite{#1}}

\usepackage{booktabs}
\usepackage[normalem]{ulem}
\useunder{\uline}{\ul}{}

\usepackage[rightcaption]{sidecap} 

\usepackage{subcaption}
\usepackage{makecell}
\usepackage{mathtools}
\usepackage{sidecap}
\usepackage{enumitem}
\usepackage{xcolor}
\usepackage{epsfig}
\usepackage{amsmath}
\usepackage{amssymb}

\newcommand{\vect}[1]{\boldsymbol{#1}}
\newcommand{\Label}{\mathcal{L}}
\newcommand{\Labeld}{\mathcal{L}_{\text{d}}}
\newcommand{\Labelc}{\mathcal{L}_{\text{c}}}
\newcommand{\R}{\mathbb{R}}

\usepackage{amsthm}

\newtheorem{theorem}{Proposition}
\theoremstyle{definition}

\usepackage{cite}
\usepackage{url}

\begin{document}
\title{Efficient and Flexible Sublabel-Accurate \\Energy Minimization}

\author{\IEEEauthorblockN{Zhakshylyk Nurlanov\IEEEauthorrefmark{1}\IEEEauthorrefmark{2} $\qquad$
Daniel Cremers\IEEEauthorrefmark{2} $\qquad$
Florian Bernard\IEEEauthorrefmark{1}\IEEEauthorrefmark{2}}
\IEEEauthorblockA{\IEEEauthorrefmark{1}University of Bonn $\qquad$ \IEEEauthorrefmark{2}Technical University of Munich}
}

\maketitle

\begin{abstract}
We address the problem of minimizing a class of energy functions consisting of data and smoothness terms that commonly occur in machine learning, computer vision, and pattern recognition. While discrete optimization methods are able to give theoretical optimality guarantees, they can only handle a finite number of labels and therefore suffer from label discretization bias. Existing continuous optimization methods can find sublabel-accurate solutions, but they are not efficient for large label spaces. In this work, we propose an efficient sublabel-accurate method that utilizes the best properties of both continuous and discrete models. We separate the problem into two sequential steps: (i) global discrete optimization for selecting the label range, and (ii) efficient continuous sublabel-accurate local refinement of a convex approximation of the energy function in the chosen range. Doing so allows us to achieve a boost in time and memory efficiency while practically keeping the accuracy at the same level as continuous convex relaxation methods, and in addition, providing theoretical optimality guarantees at the level of discrete methods. 
Finally, we show the flexibility of the proposed approach to general pairwise smoothness terms, so that it is applicable to a wide range of regularizations. Experiments on the illustrating example of the image denoising problem demonstrate the properties of the proposed method. The code reproducing experiments is available at \url{https://github.com/nurlanov-zh/sublabel-accurate-alpha-expansion}.
\end{abstract}
\footnotetext{To be published at ICPR 2022, \copyright{} 2022 IEEE}

\IEEEpeerreviewmaketitle

\section{Introduction}
\begin{figure}[t!]
\centering
\captionsetup[subfigure]{justification=centering}
\footnotesize
\begin{subfigure}[t]{\columnwidth}\centering
    \includegraphics[width=0.93\columnwidth]{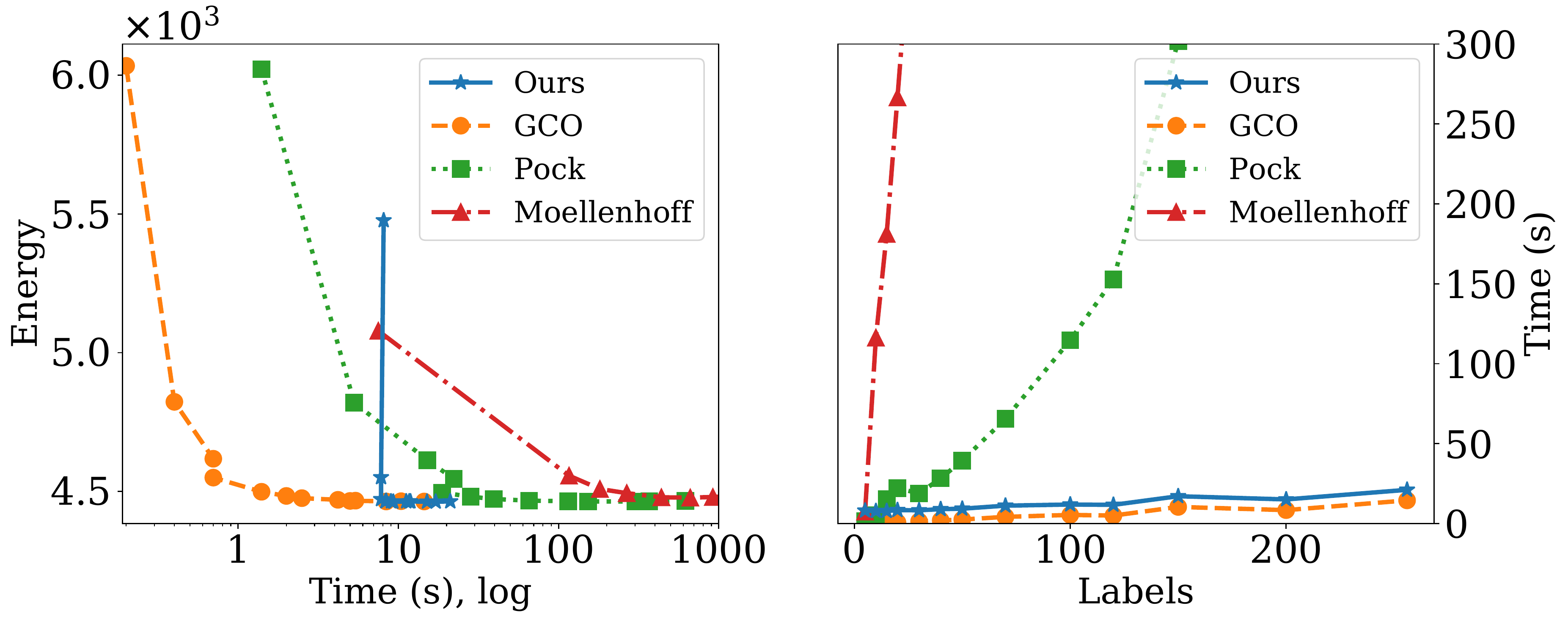}
    \end{subfigure} 
\begin{tabular*}{\columnwidth}{p{16mm}p{16mm}p{16mm}p{16mm}}
    \begin{subfigure}[t]{0.22\columnwidth}\centering
    \includegraphics[width=\columnwidth]{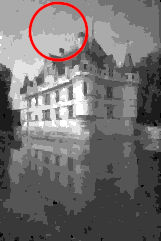}
    \caption*{Pock\\(scalable \xmark,\\sublabels \cmark)}
    \label{fig:teaser_Pock}
    \end{subfigure} &
    \begin{subfigure}[t]{0.22\columnwidth}\centering
    \includegraphics[width=\columnwidth]{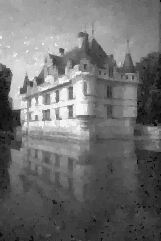}
    \caption*{Mollenhoff\\(scalable \xmark,\\sublabels \cmark)}
    \label{fig:teaser_Moellenhoff}
    \end{subfigure} &
    \begin{subfigure}[t]{0.22\columnwidth}\centering
    \includegraphics[width=\columnwidth]{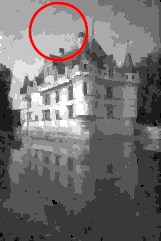}
    \caption*{GCO\\(scalable \cmark,\\sublabels \xmark)}
    \label{fig:teaser_GCO}
    \end{subfigure} &
    \begin{subfigure}[t]{0.22\columnwidth}\centering
    \includegraphics[width=\columnwidth]{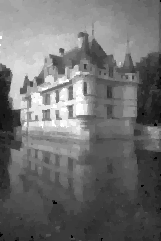}
    \caption*{\textbf{Ours}\\(scalable \cmark,\\sublabels \cmark)}
    \label{fig:teaser_ours}
    \end{subfigure} \\
\end{tabular*}
\caption{Comparison of the methods on image denoising problem. {In both of the top plots each point for a method represents a fixed number of labels 5, 10, $\ldots$, 256.
Continuous optimization methods by Pock et al.\cite{Pock2010} and Moellenhoff et al.~\cite{Mollenhoff2016} solve the problems of choosing the optimal label range and sublabel-accurate refinement simultaneously, which makes them poorly scalable to large label spaces (see top right plot). In contrast, we propose to solve the two problems sequentially, which significantly reduces the computational costs producing almost optimal results much faster than previous continuous methods (see top left plot in logarithmic time). The discrete initialization of our method (GCO~\cite{Boykov2001}) scales best w.r.t. number of labels, but suffers from label discretization bias.} Our sublabel-accurate refinement substantially reduces label discretisation artefacts (see images on the bottom for the case of $|\Label|=10$ labels), takes constant time w.r.t. label space discretization, and requires fewer number of labels than discrete method to reach plausible results.
} 
\label{fig:teaser}
\end{figure}

Many problems in machine learning, computer vision, and pattern recognition can be formulated as energy minimization over mappings $u: \Omega \rightarrow \Label$ between sets $\Omega$ and $\Label$. The energy function is constructed in such a way that the minimizing mapping constitutes certain desirable properties. Typically, such energies comprise two components, data and smoothness terms. The data term is associated with the observed data and encourages the mapping $u$ to describe the data optimally. The smoothness term is associated with some priors of the desired mapping, such as spatial smoothness or other regularizations. Hence, a respective energy minimization problem has the form
\begin{equation}
\label{eq:basic_energy}
\min\limits_{u: \Omega \rightarrow \Label} E_{\text{data}}(u) + E_{\text{smooth}}(u).
\end{equation}
Depending on the structure of the \textit{domain set} $\Omega$ and the \textit{label set} $\Label$, the optimization problem (\ref{eq:basic_energy}) can be divided into different classes. In this work we are interested in \emph{discrete-continuous} optimization problems, where the domain is represented by a discrete set of nodes $\Omega = \{1, \ldots, N\}$, and the labels of the nodes can be assigned from a \emph{continuous} set,~e.g.~{$\Label = \R$}. Hence, the mapping $u: \Omega \rightarrow \Label$ is uniquely defined by the values in $\Label$ at $N$ discrete points of $\Omega$, in other words, the mapping $u$ can be represented by the $N$-dimensional vector $\vect{u} = (u^1,\ldots, u^N) \in \Label^N$. 

In this work, we consider the Problem~\eqref{eq:basic_energy} in the following form,
    \begin{equation}
    \label{eq:discrete_energy}
    \min\limits_{\vect{u} \in \Label^N} \sum\limits_{i=1}^{N} E_{i}({u}^i) + \sum\limits_{(i, j) \in \mathcal{E}} E_{ij}({u}^i, {u}^j),
    \end{equation}
where the $E_i$ denotes the individual unary data term, and the $E_{ij}$ denotes the individual pairwise term that is defined for pairs of nodes $(i,j)$ sharing an edge in the set of edges~$\mathcal{E}$. 
Such energy formulations have a high relevance for various tasks in computer vision or pattern recognition, including image completion~\cite{komodakis2006image}, image denoising~\cite{roth2005fields}, stereo and multi-view reconstruction~\cite{kolmogorov2001computing, kolmogorov2002multi}, or optical flow estimation~\cite{Glocker2008}. 
In this context, the nodes often represent pixels (or a group of pixels) of an image, and the labels are target values from a continuous space. Moreover, the discrete domain set $\Omega$ allows to model arbitrary graph labeling problems without any restrictions on the underlying space of graph nodes so that diverse problems beyond computer vision can be addressed~\cite{Bernard2017a, paulsen2009markov, hinton2005learning, xiong2020general, nguyen2019fake}.

\section{Related work}
Since our work utilizes ideas both from fully-discrete MRF optimization, as well as from
sublabel-accurate (continuous) methods, relevant ideas from both worlds will be discussed. 

\subsection{Discrete energy minimization}
The problem formulation in a \textit{fully-discrete} setting is the most studied one in the MRF community. There are diverse variants of solutions in this direction, including belief propagation~\cite{Yedidia2001NIPS, Weiss2001}, continuous relaxations~\cite{Kannan2020}, or 
graph cuts~\cite{Boykov2001, Kolmogorov2004}. In general, the optimization of \eqref{eq:discrete_energy} is an NP-hard problem~\cite{SHIMONY1994399, Li2016ECCV}. However, for specific energy functionals, the problem can be solved efficiently. For instance, arbitrary unary potentials and convex pairwise terms are sufficient to solve Problem~\eqref{eq:discrete_energy} exactly~\cite{Ishikawa2003, schlesinger2006transforming}, or for truncated convex pairwise terms a solution with tight approximation guarantee can be found~\cite{Veksler2007}.

Similarly, the {$\alpha$-expansion} algorithm~\cite{Boykov2001} utilizes graph-cuts to solve Problem~\eqref{eq:discrete_energy} with approximation guarantees in the case of arbitrary data and \emph{metric} pairwise terms. The FastPD algorithm is a generalization of $\alpha$-expansion~\cite{Komodakis2008} to \emph{semi-metric} smoothness functionals based on a primal-dual strategy. Both of the methods are known to be extremely efficient while still guaranteeing almost optimal solutions~\cite{Szeliski2006, Nieuwenhuis2013}.
The energy functionals with more general non-convex priors (pairwise terms) are considered in the work IRGC \cite{Ajanthan_2015_CVPR}. To overcome the issue with non-convexity, it is proposed to iteratively approximate the original energy with an appropriately weighted surrogate energy that is easier to minimize. The surrogate problem at each iteration can be solved efficiently with graph cuts~\cite{Boykov2001,Ishikawa2003}. 

A major downside of the mentioned discrete approaches is that they are tailored towards finite label spaces, and that limits the applicability of such models in the context of continuous label spaces $\Label$.

\subsection{Sublabel accuracy} 
For many practical problems, sublabel accuracy is achieved by iteratively discretizing a continuous label space in a coarse-to-fine manner, and solving discrete MRFs repeatedly. For example, \citet{Bernard2017a}  discretizes the large continuous label space, the Lie group of rigid-body motions $SE(3)$, in million points at each iteration of a coarse-to-fine approach. So each iteration of the method starts with the solution from the previous iteration. 
Although this approach makes it possible to solve problems with large label spaces efficiently, it suffers from label discretization bias.

An alternative is to consider sublabel-accurate continuous approaches, which are able to assign labelings that lie \emph{in-between} discrete labels.
The problem of sublabel-accuracy in \textit{fully-continuous} optimization was explicitly tackled in the work by Moellenhoff et al.~\citet{Mollenhoff2016}. The authors propose a functional lifting approach for a piece-wise convex approximation of a data term and a fixed total variation (TV) pairwise term.
The approximation is made either via piecewise linear~\cite{Pock2010} or piecewise quadratic functions or calculated with the help of convex envelopes. Follow-up works  generalize this approach to vectorial energies~\citet{Laude2016}, total generalized variation~\citet{strecke2018sublabel}, or more general pairwise terms~\citet{Mollenhoff2017}.

The discrete domain and continuous label space are considered in discrete-continuous sublabel-accurate methods~\citet{Zacha,Zach2017,Fix2014}. 
Interestingly, the discrete-continuous method \cite{Fix2014}, the generalization of \cite{Zacha}, clearly corresponds to the continuous approach~\citet{Mollenhoff2017}, but on discrete domain. In both of the approaches, the authors work with the dual formulation of the convex relaxation and showcase the relationship between the choice of representation of dual variables and the convexification of primal energies. Therefore, these methods share the same computational complexity and lack of optimality guarantees for common functionals.

As illustrated in Fig.~\ref{fig:teaser}, the computational cost of the sublabel-accurate approaches grows \emph{dramatically} with the increase in the number of label ranges, i.e.~the label space discretization. This is mainly because these convex relaxation methods have to jointly solve two problems: (i) the choice of the label range, and (ii) the sublabel-accurate convex optimization on the chosen interval. For problems with large label spaces that rely on a precise discretization, these methods become computationally infeasible~\cite{Zacha}. Another drawback of existing sublabel-accurate methods is that they do not provide optimality guarantees in cases of nontrivial (nonlinear) data term approximations, not even for problems in which analogous discrete approaches guarantee optimality. For example, Problem~\eqref{eq:discrete_energy} with arbitrary data cost and convex pairwise cost can be solved exactly via discrete optimization~\cite{Ishikawa2003, schlesinger2006transforming}, however, the sublabel-accurate methods with non-linear data term approximation~\cite{Zacha, Mollenhoff2016} do not guarantee the optimality and do not even reach it on practice~\cite{Zacha}.

\textbf{Our contribution }
In this work, we address the shortcomings of existing works related to discretization bias and infeasible computational costs in large label spaces. To this end, we propose an efficient and sublabel-accurate algorithm for continuous-valued MRFs of the form in Problem~\eqref{eq:discrete_energy}.
We summarize our  main contributions as follows:
\begin{itemize}
    \item We propose a method for sublabel-accurate optimization of continuous-valued MRFs that does not suffer from label discretization bias while preserving the optimality guarantees of discrete models.
    \item Our method is orders of magnitude more efficient compared to previous approaches~\cite{Zacha,Pock2010, Mollenhoff2016}, both in terms of time and memory consumption, while at the same time producing solutions of similar or better quality. 
    \item The choice of data- and pairwise-costs is widely flexible and does not affect the efficiency of our method, in contrast to previous sublabel-accurate methods.
    \item Our approach is especially well-suited for problems with very large label spaces and a fine discretization.
    \item As an illustrating example, for the problem of image denoising we experimentally demonstrate that our method gives high-quality and sublabel-accurate results while requiring less time compared to competitors, see Fig.~\ref{fig:teaser}.
\end{itemize}

\section{Our Sublabel-Accurate Energy Minimization}

\begin{algorithm}[hbt!]
\caption{Proposed sublabel-accurate energy minimization}
\label{alg:proposed}
\begin{algorithmic}[1]
\REQUIRE{Continuous label space $\Labelc$; its discretization in a finite number of points $\Labeld = \{\ell_1, \ldots, \ell_L\} \subset \Labelc $; the values of functions $E_i, E_{ij}$ from Problem~\eqref{eq:discrete_energy}} in $\Labeld$ and $\Labeld \times \Labeld$
\ENSURE{vector $\vect{u} \in \Labelc^N$} which minimizes \eqref{eq:discrete_energy}
 \STATE $\vect{u}_\text{d} \leftarrow D[\{E_i(\Labeld)\}, \{E_{ij}(\Labeld)\}; \Labeld]$;  \quad 
 \textit{\# $D$ is an arbitrary discrete solver}
 \STATE $\Gamma \leftarrow$ Neighborhoods($\vect{u}_\text{d}$);  \quad \quad \quad  
 \textit{\# $\Gamma \subset \Labelc^N$-- label range} 
 \STATE Approximate $E_i, E_{ij}$ in-between labels $\Labeld$ on $\Gamma$ with one of the methods \eqref{eq:LP_local_relax_MRF}, \eqref{eq:QP_local_relax_MRF}, \eqref{subeq:Full_local_relax_constraint_1};
 \STATE $\vect{u} \leftarrow$ Solve the corresponding convex problem;
 \RETURN $\vect{u}$
\end{algorithmic}
\end{algorithm}

We propose to solve the difficult continuous labeling problem \eqref{eq:discrete_energy} by dividing it into two easier subproblems: (i) discrete label range selection and (ii) continuous sublabel-accurate optimization on the chosen label range. In the first step, we utilize an off-the-shelf discrete MRF solver.
In the second step, we solve the continuous convex problem with the convex
approximation of the original (potentially non-convex) energy on the selected label intervals.
The individual label intervals are defined as the neighborhoods (closed balls) of the initial discrete solution. The {convex} optimization problem has a constant time and memory complexity w.r.t. the label space discretization. Thus, together with the efficient discrete solver, our overall method becomes highly scalable to large continuous label spaces. The proposed approach is summarized in Algorithm \ref{alg:proposed}.

In the following, we first summarize the existing discrete methods, discuss their properties and use cases. Next, we explain the proposed continuous local refinement in detail.

\subsection{Discrete energy minimization in a nutshell}
\label{section:discrete_optimization}
We assume that the Problem \eqref{eq:discrete_energy} is given by the finite number of function values $E_i, E_{ij}$. In other words, the functions $E_i, E_{ij}$ are only evaluated on the discretization ${\Labeld = \{\ell_1, \ldots, \ell_L\} \subset \Labelc}$ of the continuous label space $\Labelc$.

There are many efficient discrete algorithms for common energy functionals. For example, the problems with arbitrary data cost $E_i$ and convex priors $E_{ij}$ can be efficiently solved to the \emph{exact} optimum via graph-cuts~\cite{Ishikawa2003}. If the constraint on pairwise terms is relaxed to metric functions, the problem can still be efficiently solved via graph-cut optimization algorithm, $\alpha$-expansion or \emph{GCO}, with tight optimality guarantee~\cite{Boykov2001},~\cite{Veksler2007},~\cite{Komodakis2008}. More general priors, such as semi-metric or non-convex functions, can also be tackled with approximately optimal algorithms FastPD~\cite{Komodakis2008} and IWGC~\cite{Ajanthan_2015_CVPR} correspondingly.

In this work, we propose to use $\alpha$-expansion (GCO) method as a default discrete solver, because it is known to be extremely efficient \cite{Szeliski2006} on many practical problems while giving almost optimal results. Moreover, the $\alpha$-expansion consumes constant memory w.r.t. label space discretization. It iterates over expansion moves, and the single move requires $O(N)$ memory, where $N = |\Omega|$ is the number of nodes. So the whole algorithm also takes $O(N)$ memory and does not depend on label space discretization $L=|\Labeld|$.

To emphasize that the choice of the discrete optimization algorithm is 
flexible and can be adjusted to the considered problem, we also use the exact discrete solver~\cite{Ishikawa2003,schlesinger2006transforming} in the experiments. It constructs graph-cut problem for all labels simultaneously, so the method requires $O(N L)$ memory.

\subsection{Proposed continuous refinement}
\label{section:continuous_refinement}
Assuming that the label space is discretized by a finite set of $L$ labels $\Labeld = \{\ell_1, \ldots, \ell_L\} \subset \Labelc$,
we define the \emph{marginalization set}
\begin{align}
    \mathcal{M}  := \bigg\{ &\vect{\phi} \in \R^Z ~\big|~ \sum_{\ell \in \Labeld} \phi_i(\ell) = 1 \quad \forall ~i \in \Omega, \label{eq:marginalization_all} \\
    & \sum_{m \in \Labeld} \phi_{ij}(\ell, m) = \phi_i(\ell) \quad \forall~ (i,j) \in \mathcal{E}, \ell \in \Labeld, \nonumber\\
    & \sum_{\ell \in \Labeld} \phi_{ij}(\ell, m) = \phi_j(m) \quad \forall ~(i,j) \in \mathcal{E}, m \in \Labeld \nonumber
    \bigg\}.
\vspace{-1mm}
\end{align}
Here, the bold $\vect{\phi} \in \R^Z$ denotes the stack of all unary $\phi_i(\ell)$ and all pairwise variables $\phi_{ij}(\ell, m)$. Based on the marginalization set, the discrete Problem~\eqref{eq:discrete_energy} with $\Label=\Labeld$ can be expressed as an integer {linear}  program (ILP)~\cite{Savchynskyy2020}, which reads as
\begin{align}
&\min\limits_{\vect{\phi} } && \sum_{i \in \Omega, \ell \in \Labeld} E_i(\ell) \phi_i(\ell) + \sum_{i\sim j, \ell, m}  E_{ij}(\ell, m)\phi_{ij}(\ell, m) \notag \\
 &~\quad\text{s.t. } && \quad\vect{\phi} \in \mathcal{M} \cap \{0,1\}^Z.
 \label{eq:discrete_MRF}
\vspace{-2mm}
\end{align}

A common way to address Problem~\eqref{eq:discrete_MRF} is based on a continuous LP-relaxation~\cite{Kannan2020}, which relaxes the binary constraint 
    $ \vect{\phi} \in \mathcal{M} \cap [0,1]^Z.$
The resulting labeling $\vect{u}$ can be obtained as the expectation 
   $ \vect{u} \in \Labelc^N,  {u^i = \sum_{\ell \in \Labeld} \phi_i(\ell) \cdot \ell}. $
In general, the LP relaxation is not tight and is not efficient. The number of continuous variables $\vect{\phi} \in \mathbb{R}^Z$ is equal to ${Z = Z_{\text{LP}}:= NL + |\mathcal{E}|L^2}$. So for the grid structured graph the memory consumption will be $O(NL^2)$.

\subsubsection{Linear data term with Marginalization constraints (LM)}
\label{subsection:local_MRF_relaxation}
To address the issues of LP-relaxation, we propose to add a local sparsity-aware constraint 
    $\vect{\phi} \in \mathcal{S}$,
where $\mathcal{S}$ imposes the predefined sparsity of $\vect{\phi}$ due to the selected label range ${\Gamma = \Gamma_1 \times \cdots \times \Gamma_N} \subset \Labelc^N$,
\begin{align}
    \mathcal{S} = \big\{ \vect{\phi} \in \R^{Z_\text{LM}} ~|~ 
    &\phi_i(\ell) = 0~\forall ~ \ell \in \Labeld \setminus \Gamma_i, \\
    &\phi_{ij}(\ell, m) = 0~ \forall~ \ell \in \Labeld \setminus \Gamma_i, m \in \Labeld \setminus \Gamma_j \big\}. \nonumber
\end{align}
The label range $\Gamma$ is induced by the initial discrete solution ${\vect{u}_{\text{d}} \in \Labeld^N}$, ${\Gamma_i = [u^i_{\text{left}}, u^i_{\text{right}}] \ni u^i_{\text{d}}}$. 
With that, the marginalization constraints become less flexible, thus significantly tighter. Overall, our \emph{local LP-relaxation} reads
\begin{align}\label{eq:LP_local_relax_MRF}
&\min\limits_{\vect{\phi} } && \sum_{i \in \Omega, \ell \in \Labeld} E_i(\ell) \phi_i(\ell) + \sum_{i\sim j, \ell, m}  E_{ij}(\ell, m)\phi_{ij}(\ell, m) \notag \\
 &~\quad\text{s.t. } && \quad\vect{\phi} \in \mathcal{M} \cap [0,1]^Z \cap \mathcal{S}, \tag{LM}
\vspace{-1mm}
\end{align}

Moreover, with the proposed sparsifying constraint only three variables, namely $\phi_i(u^i_{\text{left}}), \phi_i(u^i_{\text{d}}), \phi_i(u^i_{\text{right}})$, are not constant zeros for the node $i$. Thus, the total number of non-zero variables becomes $Z_{LM} := 3N + 3^2 |\mathcal{E}|$. 
It means that together with the GCO initialization ($O(N)$~memory cost), it leads to the total memory requirement of $O(N)$ for grid structured graphs for the whole proposed sublabel-accurate method (GCO+LM). 

However,  as it is shown in the results of Pock et al.~\citet{Pock2010} in Fig.~\ref{fig:teaser} and Table~\ref{table:discretized_energies}, the \emph{linear} approximation of energy function often leads to integer (i.e.~non-sublabel) solutions.  

\subsubsection{Quadratic data with Marginalization constraints (QM)}
\label{subsection:local_conv_data}
An alternative to the linear approximation of the energy in-between discrete labels is a local \emph{convex non-linear} approximation. For the sake of simplicity, we consider the data and pairwise terms independently, similar to~\citet{Mollenhoff2016}.

For each node $i$ we approximate the data cost $E_i$ by fitting the convex quadratic function $Q_i: \Gamma_i \rightarrow \R$. If the data cost is convex on the given label range, the approximation contains all three discrete points of the data term $E_i(\Gamma_i)$. If the data term $E_i$ is non-convex on $\Gamma_i$, a linear approximation containing the initial discrete point $E_i(u^i_\text{d})$ is used.
Thus, it is guaranteed that the continuous approximation of the discrete data term contains the initial energy-labeling point. 

The pairwise term is connected via marginalization constraint, similar to \eqref{eq:LP_local_relax_MRF}. We introduce the variable ${X \in \Gamma}$  with the additional coupling constraints ${\sum_{\ell \in \Gamma_i} \phi_i(\ell) \cdot \ell = X_i}$ for all $i \in \Omega$. With that, the total number of non-zero variables remains $O(N)$. 
Overall, the new formulations reads
\begin{align}\label{eq:QP_local_relax_MRF}
&\min\limits_{\vect{\phi}, X } && \sum_{i \in \Omega}  Q_i(X_i) \phi_i(\ell) + \sum_{\mathclap{\substack{i\sim j\\\ell \in \Gamma_i, m \in \Gamma_j}}}  E_{ij}(\ell, m)\phi_{ij}(\ell, m) \tag{QM} \\
 &~\quad\text{s.t. } && \vect{\phi} \in \mathcal{M} \cap [0,1]^Z \cap \mathcal{S},  X_i \in \Gamma_i,\notag
\sum_{\ell \in \Gamma_i} \phi_i(\ell) \cdot \ell = X_i. 
\end{align}

\subsubsection{Quadratic data and Linear pairwise terms (QL)}
In previous formulation we did not make any assumptions regarding the pairwise term $E_{ij}$.
However, as shown in~\cite{Ishikawa2003,Mollenhoff2016,Zacha,Fix2014}, restricting the structure of the pairwise terms can further reduce memory and computational costs. That is why in this formulation we assume that the pairwise energy ${E_{ij}: \Gamma_i \times \Gamma_j \rightarrow \R_{\geq 0}}$  is a convex kernel ${k_{ij}(\ell_1, \ell_2)} = {k_{ij}(|\ell_1 - \ell_2|)}$ on the given label range $\Gamma_i \times \Gamma_j$. Moreover, for computational benefits we assume that it is a linear kernel ${k_{ij}^{l_1}(\ell_1, \ell_2)} = \kappa_{ij} \cdot |\ell_1 - \ell_2|$, 
where $\kappa_{ij} \geq 0$ is a fitted constant for each edge $(i, j) \in \mathcal{E}$.

In the case of an L1 total variation (TV) regularizer, {the linear kernel approximation (denoted QL) becomes an exact approximation of the pairwise term, and it gives rise to an efficient linear programming formulation~\cite{boyd2004convex, Lofberg2004}.} 
In summary, we obtain the  convex optimization problem
\begin{subequations}\label{eq:Full_local_relax_MRF}
\begin{align}
\min\limits_{X}& \sum_{i} Q_i(X_i) + \sum_{i\sim j} k_{ij}(|X_i  - X_j|) \notag \\
\text{s.t. } & \quad X_i \in \Gamma_i. \tag{QL} \label{subeq:Full_local_relax_constraint_1}
\end{align}
\vspace{-1mm}
\end{subequations}
In this formulation, we locally approximated the data and pairwise terms on the given label ranges $\Gamma$ and $\Gamma \times \Gamma$. Thus, the problem contains ${Z_{\text{QL}} := N}$ variables, which directly form the desired sublabel-accurate labeling, i.e. $\vect{u} = X$.

\textbf{Summary }
To sum up, all three sublabel-accurate refinement formulations~\eqref{eq:LP_local_relax_MRF}, \eqref{eq:QP_local_relax_MRF}, \eqref{subeq:Full_local_relax_constraint_1} are built on the selected label ranges ${\Gamma = \Gamma_1 \times \cdots \times \Gamma_N} \subset \Labelc^N$, which in turn depend on the initial discrete solution $\Gamma = \Gamma(\vect{u}_\text{d})$. Also, each of the convex problems requires only $O(N)$ number of variables. 
Therefore, the whole proposed pipeline initialized with the efficient discrete solver GCO requires only $O(N)$ memory.
For comparison, the memory costs for previous sublabel-accurate algorithms \cite{Pock2010,Mollenhoff2016,Zacha} are ${O(NL)}$.

As it was shown in~\cite{Pock2010,Mollenhoff2016} the linear energy approximation~~\eqref{eq:LP_local_relax_MRF} often leads to non-sublabel-accurate results. 
\eqref{eq:QP_local_relax_MRF} and \eqref{subeq:Full_local_relax_constraint_1} approximate the data term $E_i$ via locally convex quadratic functions, and that often leads to sublabel-accurate results. However, it is worth noting that in the case of non-convex data terms, the approximation becomes linear and the solutions do not enjoy the sublabel improvement. 
The problem~\eqref{eq:QP_local_relax_MRF} allows a flexible choice of pairwise terms $E_{ij}$. The~\eqref{subeq:Full_local_relax_constraint_1} formulation restricts the pairwise functions to be convex or even linear, which makes it less flexible, but more computationally efficient.

\textbf{Theoretical optimality guarantees } Following, we formalize the optimality properties of the proposed method.
\begin{theorem}[Optimality preservation]
\label{proposition:optimality_preservation}
Let ${\Labeld = \{\ell_1, \ldots, \ell_L\} \subset \Labelc}$ be the discretization of the continuous label space $\Labelc$.
The functions $F: \Labeld \mapsto \mathbb{R}$, $G:\Labeld \times \Labeld \mapsto \mathbb{R}$ are given. Let $\vect{u}_\text{d} \in \Labeld$ be an arbitrary discrete labeling, and ${\Gamma}(\vect{u}_\text{d}) \subset \Labelc$ be its neighborhood (ball) in continuous label space. The energy value at the given point is equal to $E_\text{d} := F(\vect{u}_\text{d}) + G((\vect{u}_\text{d} \times \vect{u}_\text{d}))$. 

If the functions $\tilde{F}: \Gamma \mapsto \mathbb{R}$, $\tilde{G}:\Gamma \times \Gamma \mapsto \mathbb{R}$ are continuations of the discrete functions $F, G$ on continuous label range $\Gamma$, i.e. $\tilde{F}\big|_{\Labeld} = F\big|_\Gamma$,  $\tilde{G}\big|_{\Labeld} = G\big|_\Gamma$, then
\begin{align}
\min\limits_{\vect{u} \in \Gamma(\vect{u}_\text{d})}  \quad \tilde{F}(\vect{u})  + \tilde{G}((\vect{u} \times \vect{u})) = E^* \leq E_\text{d}  
 \label{eq:proposition_1}
\vspace{-2mm}
\end{align}
\end{theorem}
\begin{proof}
Since $\vect{u}_\text{d} \in \Gamma(\vect{u}_\text{d})$, and ${\tilde{F}(\vect{u}_\text{d}) = {F}(\vect{u}_\text{d})}$, ${\tilde{G}((\vect{u}_\text{d} \times \vect{u}_\text{d})) = {G}((\vect{u}_\text{d} \times \vect{u}_\text{d}))}$, then the energy value $E_\text{d} = F(\vect{u}_\text{d}) + G((\vect{u}_\text{d} \times \vect{u}_\text{d}))$ is among the optimized set.
\end{proof}

In the proposed sublabel-accurate refinement methods the optimized set of energy functions enlarges to the continuous domain and always contains the original discrete labeling-energy points on the chosen label ranges, which leads to Prop.~\ref{proposition:optimality_preservation}. 
{
\begingroup
\setlength{\tabcolsep}{4pt} 
\begin{table}[t]
    \centering
        \footnotesize
    \begin{tabular}{@{}lllll@{}}
\toprule 
$|\Label|$ &
  \makecell{Sublabel, \\ Pock} &
  \makecell{Sublabel, \\ Mollenhoff} &
  \makecell{Sublabel,  \\ exact + QL} &
  \makecell{Sublabel, \textbf{ours} \\ GCO + QL } \\ \midrule
5   & 6022  (1.4 s)               & \textbf{5079}  (7.5 s)               & {\ul 5684}  (10.1 s) & 5698  (8.1 s)  \\
10  & 4820  (5.3 s)               & 4557  (116.3 s)             & \textbf{4547}  (8.6 s)  & {\ul 4550}  (7.8 s)  \\
15  & 4613  (15.2 s)              & 4509  (181.2 s)             & \textbf{4469}  (9 s)    & {\ul 4473}  (7.8 s)  \\
20  & 4545  (22.2 s)              & 4494  (266.6 s)             & \textbf{4462}  (11 s)   & {\ul 4466}  (8.4 s)  \\
30  & 4496  (18.8 s)              & 4479  (438.4 s)             & \textbf{4461}  (16 s)   & {\ul 4464}  (8.3 s)  \\
40  & 4481  (28.3 s)              & {4478}  (663.9 s) & \textbf{4461}  (25.1 s) & {\ul 4464}  (9 s)    \\
50 &
  4473  (39.4 s) &
  {4480}  (918.9 s) &
  \textbf{4460}  (38.2 s) &
  {\ul 4463}  (9.3 s) \\
100 & {\ul 4465}  (114.7 s)             & 4490  (1794.3 s)            & OOM            & \textbf{4464}  (11.9 s) \\
150 & {\ul 4464}  (301.8 s)             & TLE                         & OOM            & \textbf{4463}  (17.1 s) \\
200 & {\ul 4465}  (368.4 s)             & TLE                         & OOM            & \textbf{4463}  (15.1 s) \\
256 & {\ul 4466}  (620.4 s)             & TLE                         & OOM            & \textbf{4464}  (21.1 s) \\
\bottomrule
\end{tabular}
    \caption{
    Comparison of sublabel-accurate methods in terms of the accuracy ${E = E_{\text{data}} + E_{\text{smooth}}}$  and efficiency (elapsed time) for robust image denoising. 
    The \textbf{bold} numbers represent the smallest energies for the given number of labels $|\Label|$ in each row, \underline{underlined} numbers the second smallest. 
    Initialized by the efficient $\alpha$-expansion algorithm our sublabel-accurate method (GCO + QL) achieves almost optimal energies without suffering from memory (OOM -- out of memory) and time (TLE -- time limit exceeded) issues. {Our proposed method is computationally more efficient than previous sublabel-accurate approaches~\cite{Pock2010, Mollenhoff2016}, but produces comparable or even better results, and scales best to large label spaces.}
}
\label{table:image_denoising}
\end{table}
\endgroup
}

\section{Experiments}

{
\begingroup
\setlength{\tabcolsep}{3.5pt} 
\begin{SCtable*}[][t]
\centering
\footnotesize
\begin{tabular}{@{}lllllll@{}}
\toprule
$|\Label|$ &
  \makecell{Discrete, \\exact} &
  \makecell{Discrete, \\GCO} &
  \makecell{Discretized,  \\ exact + QL} &
  \makecell{Discretized,  \\ GCO+ QL} &
  \makecell{Discretized, \\ Pock} &
  \makecell{Discretized, \\ Mollenhoff} \\ \midrule
5   & \textbf{6023}  & {\ul 6034}  & \textbf{6023}  (5.6\%) & {\ul 6034}  (5.6\%) & \textbf{6023}  (0.02\%)          & 6058  (16.2\%) \\
10  & \textbf{4821}  & {\ul 4823}  & \textbf{4821}  (5.7\%) & {\ul 4823}  (5.7\%) & \textbf{4821}  (0.02\%)          & 4874  (6.5\%)  \\
15  & \textbf{4614}  & {\ul 4618}  & \textbf{4614}  (3.1\%) & {\ul 4618}  (3.1\%) & \textbf{4614}  (0.02\%)          & 4663  (3.3\%)  \\
20  & \textbf{4546}  & {\ul 4550}  & \textbf{4546}  (1.9\%) & {\ul 4550}  (1.9\%) & \textbf{4546}  (0.02\%)          & 4584  (1.9\%)  \\
30  & \textbf{4496}  & {\ul 4499}  & \textbf{4496}  (0.8\%) & {\ul 4499}  (0.8\%) & \textit{\textbf{4497}}  (0.02\%) & 4512  (0.7\%)  \\
40  & \textbf{4481}  & {\ul 4484}  & \textbf{4481}  (0.5\%) & {\ul 4484}  (0.5\%) & \textit{\textbf{4482}}  (0.02\%) & 4498  (0.4\%)  \\
50  & \textbf{4473}  & {\ul 4476}  & \textbf{4473}  (0.3\%) & {\ul 4476}  (0.3\%) & \textit{\textbf{4474}}  (0.02\%) & 4493  (0.3\%)  \\
100 & OOM            & {\ul 4467}  & OOM                     & {\ul 4467}  (0.07\%) & \textit{\textbf{4466}}  (0.02\%) & 4493  (0.07\%)  \\
150 & OOM            & {\ul 4465}  & OOM                     & {\ul 4465}  (0.04\%) & \textit{\textbf{4464}}  (0.0\%) & TLE             \\
200 & OOM            & {\ul 4464}  & OOM                     & {\ul 4464}  (0.02\%) & 4465  (0.0\%)                   & TLE             \\
256 & OOM            & {\ul 4464}  & OOM                     & {\ul 4464}  (0.0\%) & 4466  (0.0\%)                   & TLE              \\ \bottomrule
\end{tabular}
\caption{Energies of \emph{discretized} sublabel-accurate solutions and of discrete methods. 
Our sublabel-accurate refinement (QL) preserves the optimality of a discrete initialization, while at the same time improves the results to sublabel accuracy {(see relative improvement of energy due to sublabel-accuracy in parentheses).
The method by Pock et al. also shows discrete optimality, but it does not find better sublabel-accurate solutions (0.02\% improvement). Mollenhoff's approach \cite{Mollenhoff2016} does not provide any optimality, including discrete.}
} 
\label{table:discretized_energies}
\end{SCtable*}
\endgroup
}

We experimentally evaluate our method on the problem of robust image denoising with a non-convex data term. Moreover, we examine the optimality of sublabel-accurate methods. Finally, we demonstrate the applicability of our approach to {generic} non-convex pairwise terms. 
We compare our method against various existing approaches, including
\begin{itemize}
    \item the approximately optimal (for metric priors) and efficient discrete method (`GCO') \citet{Boykov2001},
    \item {the globally optimal (for convex priors) discrete method (`exact') \citet{schlesinger2006transforming},}
    \item {the} globally optimal (for convex priors) sublabel-accurate method by Pock et al.~\citet{Pock2010}, and
    \item the sublabel-accurate method
 by {Mollenhoff et al.~\citet{Mollenhoff2016}}. 
\end{itemize}
Since our proposed approach allows arbitrary discrete initializers, we demonstrate the sublabel-accurate refinement on both `exact' and `GCO` solutions. However, since we aim for efficiency together with optimality, we chose the $\alpha$-expansion algorithm (GCO) for the discrete optimization step as a default initialization method. The continuous convex optimization is implemented using the YALMIP
library~\cite{Lofberg2004}. Overall, the {(GCO+QL)} is the default combination for our approach, if not specified otherwise. All the experiments were run on a 4-core Intel Core i7-7700HQ CPU running at 2.80GHz clock speed with 16 GB RAM. 
A single Nvidia GeForce GTX 1050 GPU with 4GB of memory was used for reproducing the fully continuous approaches by \citet{Pock2010} and \citet{Mollenhoff2016}.

\subsection{Image denoising}
\label{subsec:image_denoising}
We would like to show the properties of the proposed method on the problem of image denoising. This problem is a common benchmark for MRFs, and it allows for the use of continuous label space $\Labelc$, which represents the pixel intensity.
We first consider a non-convex robust data term and a convex pairwise term, so that the problem can be solved to the global optimality~\citet{schlesinger2006transforming}.
For the data term we use robust truncated quadratic costs
\begin{equation}
\label{eq:trunc_rof}
    E_{i}(u^i) = \frac{\beta}{2} \min \left\{(u^i - f(x_i))^2, \nu \right\}\,,
\end{equation}
and for the pairwise term we use the L1 total variation {(L1-TV)} with weight ${\lambda=0.6}$ given by
\begin{equation}
\label{eq:l1_TV}
    E_{ij}(u^i, u^j) = \lambda  |u^i - u^j|.
\end{equation}

We generate a denoising problem instance by degrading the input image 
using both additive Gaussian noise ${\mathcal{N}(0, 0.05)}$ and a salt and pepper noise ($p=0.25$) at the same time. The parameters in \eqref{eq:trunc_rof} were chosen as $\beta = 25, \nu = 0.025$, following the protocol from previous work~\cite{Mollenhoff2016}. 

Our method (GCO+QL) achieves almost optimal sublabel-accurate results with a small number of labels, comparable to previous fully continuous methods (see Fig.~\ref{fig:teaser}, Table~\ref{table:image_denoising}), but at significantly reduced time. The previous sublabel-accurate method by Mollenhoff et al.~\citet{Mollenhoff2016} requires a very large 
a number of iterations to converge in the case of finer label space discretization. 
In contrast to previous sublabel-accurate methods, our approach (GCO+QL) scales to large label spaces with fine discretization while having the same time complexity w.r.t. the number of labels as the initial discrete optimizer (see top-right plot in Fig.~\ref{fig:teaser}). In addition, due to the flexibility of the discrete initialization, our approach with the exact discrete solution (exact+QL) finds the lowest sublabel-accurate energy (\textbf{4460}) among all the methods. However, the memory costs of the exact discrete solver (discussed above) do not allow its usage for the larger number of labels.

\subsection{Experimental optimality preservation}
\label{subsec:discrete_optimality}
Since the problems with submodular pairwise terms can be solved exactly with {the} discrete solver~\cite{schlesinger2006transforming, Ishikawa2003}, we can estimate the optimality of {all} sublabel-accurate methods on the previous image denoising problem with a convex prior. Since in our problem formulation, the energy values $E_i, E_{ij}$ are not given in-between labels $\Labeld$, its various approximations are equally possible. So we can only evaluate the unique global optimum for the discrete formulation.
To compare sublabel-accurate methods with the discrete optimum, we propose to compare the energies of the \emph{discretized} sublabel-accurate solutions.
For this, we round {the sublabel-accurate solutions} $\vect{u} \in \Labelc^N$ to the discrete label space $\Labeld$, analogous to the real-world procedure of representing images with a fixed color depth (bits per pixel). 

In Table~\ref{table:discretized_energies} we report energies after discretizing sublabel-accurate solutions together with optimal (exact) and suboptimal (GCO) discrete solutions. In parentheses (5.6\%) we give the relative improvement of energy due to the sublabel-accuracy of a given method (see the corresponding sublabel-accurate energies in Table~\ref{table:image_denoising}).
The linear method by Pock~\citet{Pock2010} shows the optimal results {after discretization} (up to minor numerical errors). However, it does not find noticeably better sublabel-accurate solutions (0.02\% improvement). The reason is that {their} interpolation between discrete labels is linear, which is known to cause discrete label bias. The convex relaxation method by Mollenhoff~\citet{Mollenhoff2016} uses quadratic data term approximation, therefore, finds smooth sublabel-accurate solutions. However, its discretization is far from optimal, as this method does not provide any theoretical optimality guarantees. The issue with the absence of experimental optimality in the case of convex priors (such as total variance) and {a} non-linear data term approximation was also reported in the sublabel-accurate discrete-continuous method by Zack et al.~\citet{Zacha}. 

In contrast, our sublabel-accurate solution preserves the optimality of the original discrete solution. It can be observed by comparing the corresponding columns of the Table~\ref{table:discretized_energies}. More importantly, our sublabel refinement also leads to a better sublabel-accurate solution, and the improvement is larger for the coarser discretization of the label space.

\begin{figure}[t]
\centering
\captionsetup[subfigure]{justification=centering}
\footnotesize
\begin{tabular*}{\columnwidth}{p{17mm}p{17mm}p{17mm}p{17mm}}
    \begin{subfigure}[t]{0.23\columnwidth}\centering
    \includegraphics[width=\columnwidth]{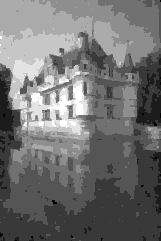}
    \caption*{GCO, ${k{=}1}$ \\ PSNR = 23.93 }
    \end{subfigure} &
    \begin{subfigure}[t]{0.23\columnwidth}\centering
    \includegraphics[width=\columnwidth]{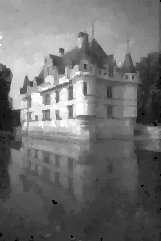}
    \caption*{Ours QM,\,${k{=}1}$ \\ PSNR = 25.13 }
    \end{subfigure} &
    \begin{subfigure}[t]{0.23\columnwidth}\centering
    \includegraphics[width=\columnwidth]{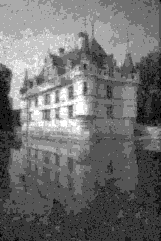}
    \caption*{GCO, ${k{=}2}$ \\ PSNR = 23.43 }
    \end{subfigure} &
    \begin{subfigure}[t]{0.23\columnwidth}\centering
    \includegraphics[width=\columnwidth]{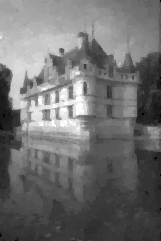}
    \caption*{Ours QL, ${k{=}2}$ \\ PSNR = 25.59 }
    \end{subfigure} 
\end{tabular*}
\caption{Results obtained by GCO and by our method with the non-convex priors  \eqref{eq:truncated_prior} for  image denoising ($|\Label|$=10). While the methods by \citet{Pock2010} and \citet{Mollenhoff2016} are not applicable in this setting, our method is flexible as it can handle non-convex priors.
} 
\label{fig:non_conv_priors}
\end{figure}

\subsection{Flexible non-convex priors}
\label{subsec:image_denoising_non_conv_priors}
As a proof of concept, we also show the applicability and efficiency of our method to arbitrary pairwise terms. In contrast, the previous sublabel-accurate methods \citet{Pock2010,Mollenhoff2016} require convex pairwise costs, {and are thus not applicable in this setting}. 
Our method with the marginalization constraint (QM) is flexible to the choice of pairwise cost, is computationally efficient, and, unlike the approach by~\citet{Zacha}, requires only a constant number of variables per node and edge.

The use of non-convex priors for image denoising is shown in Fig.~\ref{fig:non_conv_priors}. Here, we consider  non-convex truncated linear and truncated quadratic costs, i.e. 
\begin{equation}
\label{eq:truncated_prior}
    {E_{ij}(\ell, m) = \lambda \cdot \min(|\ell - m|^k, T),} \quad k \in \{1, 2\}.
\end{equation}
Parameters for truncated linear cost are ${T=0.6}$, ${\lambda=0.6}${, $k=1$}, and for truncated quadratic ${T=0.7}$, ${\lambda=3}${, $k=2$}. 
Due to the locality of the approximation, our local linear pairwise term approximation (QL) also produces reasonable results with truncated quadratic costs (Fig.~\ref{fig:non_conv_priors}, last column). However, to supply the optimality guarantees, the corresponding discrete solvers, such as~\cite{Veksler2007,Ajanthan_2015_CVPR}, that can deal with non-convex priors, must be used. Our sublabel refinement in \eqref{eq:QP_local_relax_MRF} formulation will preserve the optimality and will produce sublabel-accurate solutions.

\section{Future Work}
Since our method relies on a local refinement, it is dependent on the initialization.
Finding a trade-off between the initial discretization of the label space and the optimality of the results is an open research problem and an interesting direction for future work. Generalizing our approach to vector- and manifold-valued label spaces, e.g.~for 3D shape-to-image matching~\cite{Bernard2017a}, is a promising future direction.

\section{Conclusion}
We proposed an efficient sublabel-accurate method for energy minimization problems in the form of continuous-valued MRFs. 
We demonstrated that our method scales best to large label spaces with fine discretization, preserves the optimality achieved by the discrete model, and at the same time refines the solution based on sublabel-accurate labelings. Moreover, we showcased the flexibility of our approach regarding the choice of both data and pairwise terms, which makes it applicable to a wide range of settings.

\bibliographystyle{IEEEtran}
\bibliography{egbib}
\end{document}